\documentclass{article}

\usepackage{microtype}
\usepackage{booktabs} 

\usepackage{hyperref}

\usepackage[accepted]{icml2023}

\usepackage{xcolor}
\usepackage{graphicx}
\usepackage{dsfont}
\usepackage{amsfonts}
\usepackage{amsthm}
\usepackage{mathtools}
\usepackage{algorithm}
\usepackage{algorithmic}
\usepackage[none]{hyphenat}

\usepackage{caption}
\usepackage{subcaption}
\usepackage{tikz}
\usepackage{pgfplots}
\usepgfplotslibrary{fillbetween}
\pgfplotsset{compat=1.16}

\newenvironment{customlegend}[1][]{%
    \begingroup
    \csname pgfplots@init@cleared@structures\endcsname
    \pgfplotsset{#1}%
}{%
    \csname pgfplots@createlegend\endcsname
    \endgroup
}%

\def\addlegendimage{\csname pgfplots@addlegendimage\endcsname}
\definecolor{tcolor}{rgb}{0, 0.46484375, 0.73046875}
\definecolor{acolor}{rgb}{0.9296875, 0.46484375, 0.19921875}
\definecolor{xcolor}{rgb}{0.2980, 0.7529, 0.7490}

\theoremstyle{plain}
\newtheorem{proposition}{Proposition}

\theoremstyle{definition}
\newtheorem{definition}{Definition}
\newtheorem*{remark}{Remark}
\newtheorem*{example}{Example}

\DeclareMathOperator*{\plim}{plim}

\begin{document}

\icmltitlerunning{Arithmetic Sampling: Parallel Diverse Decoding for Large Language Models}

\twocolumn[
\icmltitle{Arithmetic Sampling: Parallel Diverse Decoding for Large Language Models}

\icmlsetsymbol{equal}{*}

\begin{icmlauthorlist}
\icmlauthor{Luke Vilnis}{equal,goog}
\icmlauthor{Yury Zemlyanskiy}{equal,goog}
\icmlauthor{Patrick Murray}{goog}
\icmlauthor{Alexandre Passos}{goog}
\icmlauthor{Sumit Sanghai}{goog}

\end{icmlauthorlist}

\icmlaffiliation{goog}{Work done while all authors were at Google Research.}

\icmlcorrespondingauthor{Luke Vilnis}{lvilnis@google.com}
\icmlcorrespondingauthor{Yury Zemlyanskiy}{urikz@google.com}

\icmlkeywords{Machine Learning, ICML}

\vskip 0.3in
]

\printAffiliationsAndNotice{\icmlEqualContribution} 

\begin{abstract}
Decoding methods for large language models often trade-off between diversity of outputs and parallelism of computation. 
Methods such as beam search and Gumbel top-k sampling can guarantee a different output for each element of the beam, but are not easy to parallelize. 
Alternatively, methods such as temperature sampling and its modifications (top-k sampling, nucleus sampling, typical decoding, and others), are embarrassingly parallel, but have no guarantees about duplicate samples. 
We present a framework for sampling according to an arithmetic code book implicitly defined by a large language model, compatible with common sampling variations, with provable beam diversity under certain conditions, as well as being embarrassingly parallel and providing unbiased and consistent expectations from the original model.
We demonstrate the effectiveness of our approach on WMT machine translation, more than halving the standard deviation when estimating expected BLEU score reward, and closing the BLEU score gap between independent sampling and beam search by up to 63\%. 
\end{abstract}

\section{Introduction}

Large language models (LLMs) based on transformers are crucial to modern natural language processing. The ability of LLMs to capture knowledge from massive pretraining datasets is useful for applications such as machine translation and predictive text \citep{t5, brown2020language, Radford2019LanguageMA} as well as for automated speech recognition \citep{martinez2021attention} and image captioning \citep{devlin2015language}. However, because of the powerful nonlinear dependencies in the architecture, options for inference are limited.

While LLM inference can be performed exactly for the case of drawing independent samples, practical systems often use inexact search---often modifications to beam search---to guarantee high-quality and diverse (either in n-gram overlap or semantic difference) samples. Search-based approaches including beam search, stochastic beam search (Gumbel top-k) \citep{kool2019stochastic}, determinantal beam search \citep{meister2021determinantal}, and others, can produce diverse samples by construction, at the cost of being difficult to efficiently parallelize, as they must examine the entire set of partial predictions (known as the \emph{beam}) at each time step.

\begin{figure}
    \centering
    \hspace*{-1.5em}
    \includegraphics[width=26em]{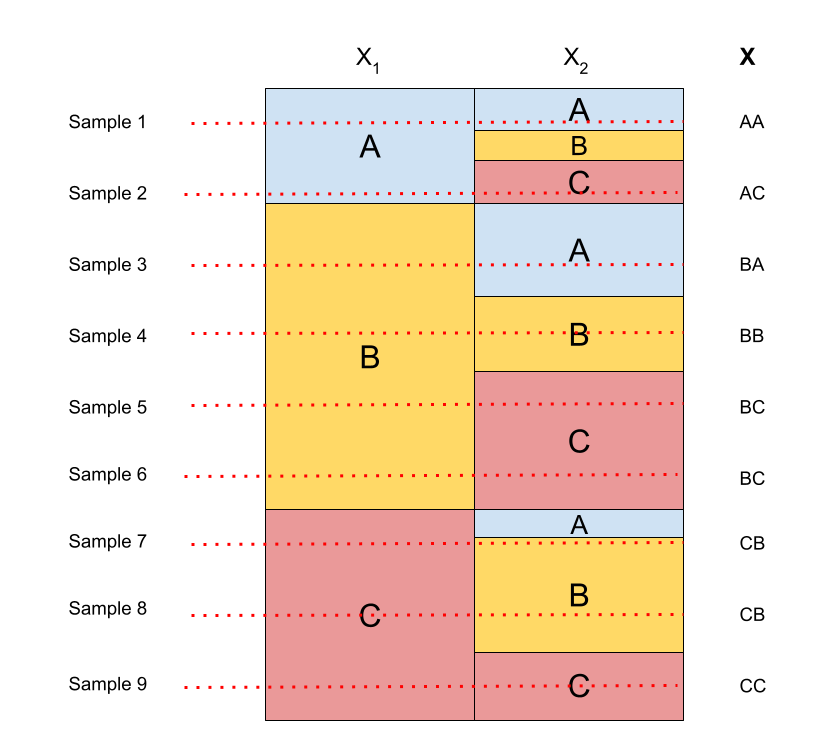}
    \caption{Sequence model over sequences of length two and a vocabulary of three symbols mapping points in the unit interval to each sequence. An even lattice of code points parallelizes decoding into diverse high-probability sequences.
    }
    \label{fig:toy_diagram}
\end{figure}

Intuitively, there seems to be a trade-off between parallelizability of a sampling algorithm and ability to guarantee non-duplicate samples.
Methods based on ancestral sampl\-ing parallelize very well as they turn a random number generation seed into a sample independently. However, high proba\-bility sequences are often generated multiple times.
Conversely, search-style methods inherently avoid generating duplicate samples but when parallelized require synchronizing across replicas to sort the candidates at each step.

In response to this, we introduce \emph{arithmetic sampling}, a technique for sampling from large language models that produces a set of \emph{non-independent} samples from the model, based on a coding scheme implicitly defined by the model. With this coding scheme, distant codes in code space represent different sequences. Further, decoding each code can be done independently from all other codes. Arithmetic sampling boasts provable beam diversity under certain conditions, produces unbiased expectations from the original model, and is embarrassingly parallel and as simple to implement as normal sampling.

In addition to analyzing bias and consistency of estimators based on arithmetic sampling, we present results on the metric properties of the codebook space and the conditions under which it improves sample diversity, as well as an analysis of the estimator variance. 

Comparing against equivalent hyperparameters for standard sampling, we show improvements of nearly 1 point of BLEU in oracle experiments for En/Fr WMT translation, closing the gap between independent sampling and beam search by up to 63\%,  reducing estimation variance by more than half and improving beam diversity. We see comparable improvements in En/Ro translation and variance reduction for ROUGE score on CNN/DailyMail summarization. We release an open-source implementation of our algorithm\footnote{Code is available at \url{https://github.com/google-research/google-research/tree/master/arithmetic_sampling}} in the popular T5X transformer library \cite{t5x}.

\section{Related Work}

This paper draws on three main threads of related work. Coding theory, diverse sampling, and quasi-Monte Carlo.

The use of latent ``codes'' to represent data has a long history in the neural network and representation learning literature, from autoencoders \citep{lecun1987phd} to sparse coding algorithms like K-SVD \citep{aharon2006k}. Rather than using a high dimensional code learned from data using an iterative algorithm or backpropagation, we design a simple one dimensional arithmetic code \citep{10.5555/1146355} post-hoc from a trained language model.

Diverse sampling inference techniques for large language models fall into two categories: techniques for producing a diverse \emph{beam} (sample of sequences), and techniques for discouraging overlap (n-gram repetition) within a single sequence. The former encompasses methods like determinantal beam search \citep{meister2021determinantal}, parallel approximate decoding \citep{cho2016noisy},  stochastic beam search \citep{kool2019stochastic}, and conditional poisson stochastic beam search \citep{meister-etal-2021-conditional}. Our method differs from (determinantal) beam search or parallel approximate decoding in that it is designed to faithfully sample from the underlying probability model in that sample means can be used to compute unbiased expectations. Unlike beam search or sampling-without-replacement based variants, our algorithm is embarassingly parallel.

Methods such as temperature sampling, top-k sampling \citep{fan2018hierarchical}, nucleus sampling \citep{holtzman2019curious}, Mirostat \citep{basu2020mirostat}, and typical decoding \citep{meister2022typical} are useful both for increasing diversity over standard beam search both across the beam and within a single long generation. These methods, and any others that modify conditional logits, are fully compatible with and complementary to our algorithm, and we provide experiments with temperature sampling and top-k sampling demonstrating improvements.

There is also work on changing the training objective using \emph{unlikelihood} \citep{welleck2019neural} or reinforcement learning \citep{lagutin2021implicit} so that standard generation schemes produce more diverse outputs, which is also orthogonal to our methods.

The final thread of related work is quasi-Monte Carlo integration. (Randomized) Quasi-Monte Carlo techniques \citep{l2016randomized} combine the adaptive anytime properties of Monte Carlo estimation with the reduced variance of lattice integration methods \citep{l2000variance}, and have been used in machine learning applications such as lowering the variance of randomized kernel approximations \citep{yang2014quasi} and neural latent variable models \citep{buchholz2018quasi}. Quasi-Monte Carlo has not been applied to the standard neural autoregressive discrete distributions we describe here, to our knowledge.

\section{Background}

\subsection{Arithmetic Coding}
\label{sec:coding}

An \emph{arithmetic code} is an optimal lossless coding scheme---that is, a coding scheme with minimal expected code length---for when the exact joint distribution of the data is known.
Given a total ordering of the items being encoded and defining $w_i = \sum_{j<i} P(X=x_j)$ the cumulative probability of item $x_i$, an arithmetic code for $x_i$ is a number in the interval $(w_i,w_{i+1})$.
To represent this code as a sequence of bits it's usual to pick a rational number in the interval $(w_i,w_{i+1})$ whose  binary fraction representation requires a small number of digits.
Larger intervals then tend to contain more numbers with short representations.
Decoding an arithmetic code $c$ requires finding the unique value of $i$ such that $w_i < c < w_{i+1}$.
In this way, codewords are assigned with a length which is inversely (logarithmically) proportional to the probability of an outcome, providing an optimal compression rate for the average message, roughly equivalent to the entropy of the distribution. 
\begin{definition}[Arithmetic codebook]
By a slight abuse of terminology, we will use the term \emph{arithmetic codebook} to refer not only to the map from symbols to binary fractional representations, but also to the map from symbols to subintervals of the unit interval, $f: V \to 2^{[0,1]}$.
\end{definition}
An example of an arithmetic codebook is he most common method for sampling from categorical distributions in practice. First one constructs a codebook assigning each symbol to a subinterval of the unit interval, and then samples uniformly at random from the unit interval and inverts the map to find the symbol sampled.

\subsection{Randomized Quasi-Monte Carlo}

A common problem when working with probability distributions is to compute the expectation of functions under the distribution.
The family of Monte Carlo algorithms is commonly used for this purpose.
A simple Monte Carlo algorithm for estimating the expectation of a function $s$ under a probability distribution is to first obtain $n$ i.i.d. samples $x_i$ from the distribution and then approximate $E[s(X)] \sim \frac{1}{n}\sum_i s(x_i) $. Without loss of generality, these methods are often formulated in terms of evaluating an expectation of a function defined on the unit hypercube.
\begin{align}
\label{eqn:expectation}
    \mathds{E}[s(X)] &= \int_{x \in [0,1]^d} s(x)dx \\&\approx \frac{1}{N}\sum_i^N s(u_i), ~~u_i \sim \text{Uniform}([0,1]^d) \nonumber
\end{align}
Because many integrals can be interpreted as expectations of functions of a uniform distribution on the unit hypercube, Monte Carlo algorithms have been fairly useful for numerical integration as 
alternatives to quadrature methods which approximate functions using grids and other regular structures, often providing higher expected accuracy when controlling for computational cost.

When dependent random variables are used in Monte Carlo estimation of probabilistic quantities, it is commonly called \emph{Randomized Quasi-Monte Carlo} (RQMC).
These methods include the use of low-discrepancy sequences, lattice rules, antithetic sampling, and stratification, among others \citep{l2016randomized}. In this work, we are most concerned with lattice-based RQMC \citep{l2000variance} methods, which use perturbed lattices. A simple lattice-based RQMC rule replaces the uniform sampling in Equation \ref{eqn:expectation} with a regular lattice of points that has been randomly shifted by a single uniform random vector.
\begin{align}
\label{eqn:rqmc-expectation}
    \mathds{E}[s(X)] \approx \frac{1}{N}\sum_i^N s(l_i + u), ~~u \sim \text{Uniform}([0,1]^d).
\end{align}

\subsection{Ancestral Sampling}
\label{sec:ancestral-sampling}
While Section \ref{sec:coding} makes it clear that given a codebook for a discrete distribution, sampling from that distribution can be done by simply generating a uniform random number, in practice  explicitly constructing such a codebook is often not feasible as for example in probabilistic sequence models the codomain includes all sequences of symbols up to some large length. Instead, one models the joint probability of a sequence of tokens as the product of the conditional probabilities of each token given all of the preceding tokens,
\begin{align}
\label{eqn:ancestral-probability}
P(X_{T},...,X_1) = \Pi_{t=0}^{T-1} P(X_{T} | X_1,...,X_{t} ).
\end{align}
Each of these conditional probability functions can then be modeled using a neural network.
Analogously, sampling in large language models is done through \emph{ancestral sampling}, wherein each token is sampled successively from the conditional probability after conditioning on all previous tokens,
\begin{align}
\label{eqn:ancestral-sampling}
x_T \sim P(X_{T} | X_1,...,X_{T-1} ).
\end{align}
\begin{definition}[Prefix of a sequence]
For a sequence of symbols $x_1, ..., x_T$, we call a contiguous subsequence $x_1,...,x_t$ for $t<T$ a \emph{prefix} of the sequence.
\end{definition}
When working with these probabilistic sequence models, it is natural to think in terms of prefixes. In fact, implicit in our definition of the ancestral sampling scheme and product-of-conditionals architecture is that it allows us to compute probabilities not only over complete sequences, but also over partial prefixes, i.e. the probability $P(X_1=x_1, ..., X_T=x_T)$ is the sum of probabilities of every sequence longer than $T$ sharing that prefix. There are two ways that practical neural sequence models distinguish between a prefix and a complete sequence. The first, common in decoder-only models, is to decode every sequence to some maximum length, and define the distribution as applying only to sequences of that length. The second is to include a special $\text{EOS}$ (end-of-sentence) token, and define a sequence as complete if it ends with $\text{EOS}$. So the prefix $(x_1,...,x_T)$ is distinguished from the sequence $(x_1,...,x_T, \text{EOS})$. Padding tokens are added to the end of the sequence after $EOS$ to make every sequence of a uniform length.

In our work we exploit this prefix structure in order to construct an alternative to ancestral sampling for these neural sequence models.

\section{Method}

The core idea of our method is to improve the diversity of samples by (1) reinterpreting an ancestral sampling scheme as defining an arithmetic codebook, where distance in code space correlates (in a sense to be made precise later) with prefix distance in sentence space, and (2) using non-IID random numbers to sample from the codebook. This allows us to guarantee that the codewords are ``far apart'' in code space, while preserving unbiasedness of our estimation. 

\subsection{Constructing the codebook}

The algorithm has a geometric flavor and will be easiest to follow while making reference to the toy example in Figure \ref{fig:toy_diagram}.
As noted in Section \ref{sec:ancestral-sampling}, in real world sequence models over a vocabulary $V$, it is impractical to explicitly construct an arithmetic codebook (a mapping from sequences to disjoint subintervals of $[0,1]$). What we will demonstrate here is that it is possible to \emph{implicitly} define an arithmetic codebook for a given sequence model such that we can (1) given a sequence or prefix, compute the corresponding interval in the codebook, and (2) given a point (a ``code'') in the unit interval, to compute the the corresponding sequence. Further, these computations can be done with complexity no greater than that of normal likelihood computation or  ancestral sampling.

Without loss of generality, we can assume the sequences all have uniform length $L$ as described in Section \ref{sec:ancestral-sampling}. Given an ordered vocabulary $V$, we use the standard dictionary ordering on $V^L$. Given two sequences $(a_1, a_2, ..., a_L)$ and $(b_1, b_2, ..., b_L)$, their ordering depends on the order between the two symbols in the first place $i$ on which the two sequences differ.

Because our dictionary ordering puts all sequences sharing a given prefix into contiguous blocks, we can define the codebook in terms of prefixes and only lazily materialize the codes for longer prefixes as we need them.

Concretely, we compute the CDF of the first token in the sequence $w_{i_1} = \sum_{j<i_1} P(X_1=v_j)$, and assign to each choice of prefix $X_1=v_{i_1}$ the subinterval $(w_{i_1}, w_{i_1+1})$. All codes for sequences starting with $v_{i_1}$ will lie in this interval. We recursively define the codebook for prefixes of length two, computing 
\begin{align*}
    \label{eqn:codebook}
    &w_{i_1i_2} = w_{i_1} + \\&~~~~P(X_1=v_{i_1}) \sum_{j<i_2} P(X_2=v_j|X_1=v_{i_1})
\end{align*}
This gives the subinterval corresponding to sequences that start with $v_{i_1}v_{i_2}$. We extrapolate the following formula:
\begin{align}
    &w_{i_1...i_L} = w_{i_1..i_{L-1}} +\nonumber\\&~~~~\sum_{j<i_{L}} P(X_1=v_{i_1},...,X_{L-1}=v_{i_{L-1}}, X_L=v_j)
\end{align}
For a given sequence or prefix $i_1...i_{L}$, we assign it to the subinterval $(w_{i_1...i_L}, w_{i_1...i_L+1})$. By inspecting this equation we can see several things: 
\begin{itemize}
    \item The intervals defined by $w_i$'s are a valid codebook for the space of sequences. The subintervals corresponding to any given prefix are disjoint from those that do not share that prefix, so every sequence ends up in a disjoint interval, and the length of each interval is exactly the probability of the sequence.
    \item The computation of a code for a given sequence has the same FLOPS as evaluating the probability of that sequence under the model using Equation \ref{eqn:ancestral-probability}. The only probabilities involved in computing the $w_i$'s all involve the conditional probabilities used to calculate a single prefix and can be computed step-by-step.
    \item Given a code point in the unit interval, discovering its subinterval requires the same FLOPS as standard ancestral sampling using Equation \ref{eqn:ancestral-sampling}. This is described in Algorithm \ref{alg:sampling-from-code} and follows the same recursive construction as used to define the codebook in Equation \ref{eqn:codebook}.
\end{itemize}
An immediate corollary of the third point is
\begin{proposition}
\label{prop:sampling_thm}
If the code point $c$ is chosen randomly from the unit interval, Algorithm \ref{alg:sampling-from-code} samples from the distribution $P(X)=P(X1,...,X_T)$.
\end{proposition}

\begin{remark}
Conditioned on a set of codes $c_1,...,c_N$, sampling using Algorithm \ref{alg:sampling-from-code} is embarrassingly parallel across sequences. Because it uses the same FLOPS and follows the same structure, we find that this has zero appreciable computational overhead in practice compared to standard sampling. We further discuss practicalities of parallel LLM inference in Section \ref{sec:parallelization-experiments} and Appendix \ref{app:parallelization}.
\end{remark}

\begin{algorithm}[tb]
   \caption{Sampling from a Code Point}
   \label{alg:sampling-from-code}
\begin{algorithmic}
   \STATE {\bfseries Input:} code point $c \in [0,1]$, ordered vocabulary $V$,\\~~~sequence model $P(X_1,...,X_T)$
   \STATE set $c_0 = c$, $X = \{\}$, $t = 0$
   \REPEAT
   \STATE set $w_i = \sum_{j < i} P(X_t = v_j | X_{< t})$\\~~~for $v_j \in V$
   \STATE set $X_t = v_i$\\~~~s.t. $w_i < c_t < w_{i+1}$
   \STATE set $(m, M) = (w_i, w_{i+1})$\\~~~s.t. $w_i < c_t < w_{i+1}$
   \STATE set $c_{t+1} = \frac{c_t - m}{M - m}$
    \STATE set $t = t + 1$
   \UNTIL{$X_{t-1} = \text{EOS}$}
   \STATE return $X$
\end{algorithmic}
\end{algorithm}

\subsection{Sampling consistency and bias}
\label{sec:consistency-and-bias}

So far we have only reproduced the standard algorithm for ancestral sampling using an additional latent uniform variable (which is similar to most practical implementations). This latent variable however lets us introduce some structure in how we pick our codes.
 
A naive codebook of maximal diversity can be obtained by dividing the unit interval in a regular lattice. That is, for $N$ codes, we pick so $c_i$ is the $i$'th quantile $i/N$. Since this is deterministic, Proposition 1 does not apply, but this gives a consistent estimator.
\begin{proposition}[Naive arithmetic sampling] 
\label{prop:naive-estimator}
Let a set of $N$ codes be picked such that $c_i=i/(N+1)$. Let $f(c_i)$ represent the sequence obtained from applying Algorithm \ref{alg:sampling-from-code} to $c_i$. Let $s$ be a function from the codomain of $X$ to $\mathbb{R}$. Then
\begin{align}
\label{eqn:naive-estimator}
\lim_{N \to \infty} \frac{1}{N} \sum_i s(f(c_i)) = \mathds{E}[s(X)],
\end{align}
that is, the estimator $\frac{1}{N} \sum_i s(f(c_i))$ is consistent.
\end{proposition}

This deterministic method is biased, however. 

\begin{example}
Consider the distribution over sequences of length one, over two symbols. Let the first symbol $A$ have probability $0.6$ and the second symbol $B$ have probability $0.4$. The naive arithmetic sampling scheme for $N=1$ would have us compute our sample mean with the code $c_1 = 0.5$, which is biased.
\end{example}

A simple technique to remove this bias it to add a single random uniform shift to each code point and taking the result modulo 1. This is a \emph{randomly shifted lattice rule} in the RQMC literature.

\begin{definition}[Arithmetic sampling]
Let $\{c_i\}$ be a set of N codes given by $\frac{i}{N+1} + b~~ \text{mod}~~ 1$ where $b \sim U(0,1)$ is a shared uniform random sample. Apply algorithm \ref{alg:sampling-from-code} to each of the codes to receive sequences $\{ x_i \}$.
\end{definition}

\begin{proposition}
\label{prop:final-estimator-unbiased}
Let $\{x_i\}$ be a set of sequences picked by arithmetic sampling. Let $s$ be a function from the codomain of $X$ to $\mathbb{R}$. Then
\begin{align*}
\mathds{E}[\frac{1}{N} \sum_i s(x_i)] = \mathds{E}[s(X)]\\
\plim_{N \to \infty} \frac{1}{N} \sum_k s(x_i) = \mathds{E}[s(X)]
\end{align*}
that is, the estimator $\frac{1}{N} \sum_i s(x_i)$ is unbiased and consistent.
\end{proposition}

Proofs of these statements can be found in Appendix \ref{app:bias-consistency}.

\subsection{Metric properties and diversity}

Intuitively, more evenly spaced out codes in code space should lead to more diverse sequences and/or more even coverage of the set of sequences. We note the following

\begin{proposition}[Monotonicity]
Consider an arithmetic codebook mapping code points in $[0,1]$ onto a space $V^L$ containing sequences of length $L$. Let $(c_3, c_4) \subset (c_1, c_2)$ be two intervals in $[0,1]$. Let $x_i = f(c_i)$ be the sequence decoded from code point $c_i$, and let $d(x_i, x_j) = L - p$ where $p$ is the length of the maximum matching prefix of the two sequences (this is a distance in a prefix tree). Then we have $d(x_1, x_2) > d(x_3, x_4)$, that is, the the function $f$ is a monotonic (order-preserving) embedding between the poset of subsets of $[0,1]$ ordered by set inclusion, and the poset of subsets of $V^L$ ordered by set radius under the prefix distance metric.
\end{proposition}

The choice of a prefix-ordered code space is of course necessitated by our construction of the codebook from the ancestral sampling scheme but it might be beneficial in practical applications, where overlapping prefixes often indicate simple or extraneous variations on the same word (such as pluralization and other inflections), while different prefixes indicate entirely different words.

We can also precisely characterize when arithmetic sampling will yield duplicate samples and prefixes.
\begin{proposition}
Arithmetic sampling with size $N$ will never sample the same prefix $x$ more than $n$ times if $P(X=x) < n/(N+1)$.
\end{proposition}
\begin{proof}
Each sequence $x$ maps to a contiguous interval with length equal to $P(X=x)$. Since each code word is $1/(N+1)$ away from any other code word, more than $n$ of them cannot fit in an  interval of that size  and thus cannot generate the same sequence.
\end{proof}
\begin{proposition}
Arithmetic sampling with size $N$ must always sample a prefix $x$ at least $n$ times if $P(X=x) > n/(N+1)$.
\end{proposition}
\begin{proof}
The code book maps each sequence $x$ to a contiguous interval with length equal to $P(X=x)$. Since each code word is $1/(N+1)$ away from its nearest other code word, at least $n$ of them must appear in an interval of length greater than $n/(N+1)$ by the pigeonhole principle.
\end{proof}

\subsection{Variance of estimator}

We would like to characterize the variance reduction properties of this method when used as an estimator. While this is not simple to characterize in general, we show a simple and reasonably realistic class of functions for which shifted lattice rules reduce variance, and offer a heuristic argument that real world quantities of interest for estimation, such as BLEU, exhibit these properties. This offers a partial theoretical explanation for the variance reduction properties we observe in our empirical experiments. 

\begin{definition}[Step function]
A function $f: [0,1] \to \mathbb{R}$ is a \emph{step function} if it is a finite linear combination of indicator functions of disjoint intervals, which cover all of $[0,1]$
\begin{align*}
    f(c) = \sum_i^N a_i \mathds{1}_{[x_i^\wedge, x_i^\vee]}(c)
\end{align*}
\end{definition}

Note that the function whose expectation we desire to estimate, $s(f(c))$, is a step function.
\begin{align*}
    s(f(c)) = \sum_i^N f(x_i) \mathds{1}_{f^{-1}(x_i)}(c) 
\end{align*}

\begin{proposition}
Let $f$ be a step function, and let the width of each interval $x_i^\vee - x_i^\wedge$ be a multiple of $1/N$. Then arithmetic sampling with $N$ samples exactly estimates this expectation with zero variance.
\end{proposition}
\begin{proof}
Subdivide every interval component of size $n/N$ (wlog) up into $n$ intervals of width $1/N$. Then 
exactly one randomly shifted lattice point will end up in each interval by the pigeonhole principle.
\end{proof}

Though it is unsurprising that taking exactly the right number $N$ will give zero variance, adding more samples may actually increase the variance a bit as some intervals end up with more points in them than others. Let's examine what happens to this term when we add one extra lattice point, that is, we integrate using $N+1$ samples.

\begin{proposition}
Let $f$ be a step function, and let the width of each interval $x_i^\vee - x_i^\wedge$ be a multiple of $1/N$. Then arithmetic sampling with $N+1$ samples has lower variance than the naive Monte Carlo estimator.
\end{proposition}

The proof of this statement is provided in Appendix \ref{app:variance-rqmc}.

As we add more lattice points, we can expect the variance to change even more, but the analysis of the addition of one extra lattice point is instructive as to how we can expect sampling on a shifted lattice to improve over naive Monte Carlo for these kinds of functions.

Practical functions on the space of sequences obviously have too many interval components, which are not perfect multiples of $1/N$, for these results to apply directly. However, if we assume that the vast majority of probability mass lies on a reasonably small number of sequences, and the value of $f$ outside of such sequences is quite low, such as when calculating BLEU score for a trained translation model, the function $s(f(c))$ should be very well approximated by a function satisfying these conditions with a reasonably fine mesh. Small differences in approximation error could still result in significant increases in estimator variance, but our empirical results suggest this is not necessarily the case for real world applications. We leave analysis of the exact approximation error and effects of additional lattice points to future work.

\section{Experiments}

\newcommand{\singleplot}[5]{
\begin{subfigure}[t]{0.3\linewidth}
\centering
\begin{tikzpicture}[scale=1.0]
    \begin{axis}[
    scale only axis,
    width=0.87\textwidth,
    height=0.9\textwidth,
    xlabel={$n$-gram diversity},
    mark size=2pt,
    solid,
    ymajorgrids=true,
    xmajorgrids=true,
    grid style=dashed,
    legend pos=south west,
    xlabel style = {font=\small},
    ylabel style = {font=\small},
    ticklabel style = {font=\small},
    max space between ticks=20
    ]
    \addlegendimage{empty legend}\addlegendentry{$N = #2$}
    \addplot [mark=none, color=black, dotted, line width=2] table[row sep=crcr] {
    #3 #5\\
    #4 #5\\
    };
    \addplot [color=tcolor,mark=square*,x filter/.expression={and(\thisrow{topk} == #1, \thisrow{num_decodes} == #2) ? x : nan}] table[x=4gram_diversity_macro,y=mean_sentence_bleu, col sep=comma]{results/t5_base_topk2_topk10_topk0_bleu_vs_diversity_t.csv};
    \addplot [color=tcolor,mark=triangle*,mark size=3pt,x filter/.expression={and(\thisrow{topk} == #1, \thisrow{num_decodes} == #2) ? x : nan}] table[x=4gram_diversity_macro,y=max_sentence_bleu, col sep=comma]{results/t5_base_topk2_topk10_topk0_bleu_vs_diversity_t.csv};
    
    \addplot [color=acolor,mark=square*,x filter/.expression={and(\thisrow{topk} == #1, \thisrow{num_decodes} == #2) ? x : nan}] table[x=4gram_diversity_macro,y=mean_sentence_bleu, col sep=comma]{results/t5_base_topk2_topk10_topk0_bleu_vs_diversity_a.csv};
    \addplot [color=acolor,mark=triangle*,mark size=3pt,x filter/.expression={and(\thisrow{topk} == #1, \thisrow{num_decodes} == #2) ? x : nan}] table[x=4gram_diversity_macro,y=max_sentence_bleu, col sep=comma]{results/t5_base_topk2_topk10_topk0_bleu_vs_diversity_a.csv};
    \end{axis}
\end{tikzpicture}
\end{subfigure}
}

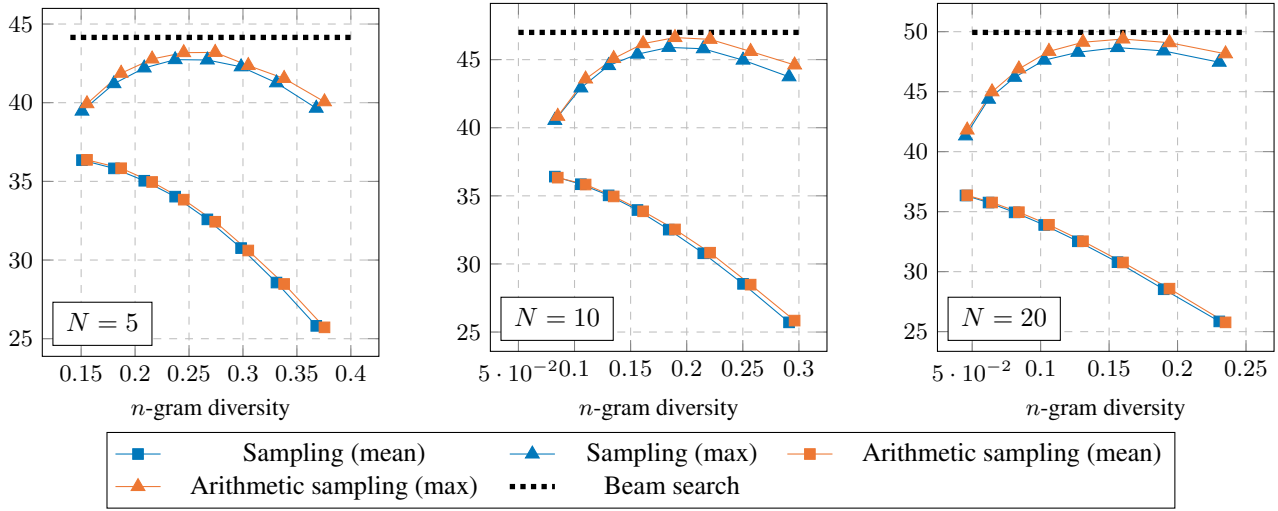
\begin{figure*}
    \centering
    \singleplot{0}{5}{0.14}{0.4}{44.15}
    \hfill
    \singleplot{0}{10}{0.05}{0.3}{47.0}
    \hfill
    \singleplot{0}{20}{0.05}{0.25}{49.94}

    \begin{tikzpicture}
        \begin{customlegend}[
            legend columns=3,
            legend style={
                align=left,
                column sep=2ex
            },
            legend entries={Sampling (mean), Sampling (max), Arithmetic sampling (mean), Arithmetic sampling (max), Beam search}
        ]
            \addlegendimage{mark=square*,solid,color=tcolor}
            \addlegendimage{mark=triangle*,mark size=3pt,solid,color=tcolor}
            \addlegendimage{mark=square*,solid,color=acolor}
            \addlegendimage{mark=triangle*,mark size=3pt,solid,color=acolor}
            \addlegendimage{mark=none,line width=2,dotted,color=black}
        \end{customlegend}
    \end{tikzpicture}    
    \caption{BLEU score on the WMT14 EnFr test set vs $n$-gram diversity.}
    \label{fig:t5_base_topk0_bleu_vs_diversity}
\end{figure*}

\let\singleplot\undefined

\newcommand{\singleplot}[5]{
\begin{subfigure}[t]{0.3\textwidth}
\centering
\begin{tikzpicture}
    \begin{axis}[
    scale=0.6,
    xlabel={#3},
    ylabel={BLEU estimate},
    mark=x,
    ymajorgrids=true,
    xmajorgrids=true,
    grid style=dashed,
    legend pos={#4},
    xlabel style = {font=\small},
    ylabel style = {font=\small},
    ticklabel style = {font=\small},
    ytick = {#5}
    ]
    \addlegendimage{empty legend}\addlegendentry{$T = #2$}
    \addplot [name path=lowerbound,color=tcolor,no markers,x filter/.expression={and(\thisrow{temperature} == #2, \thisrow{index} == #1) ? x : nan}] table[x=num_decodes,y=mean_sentence_bleu_025p, col sep=comma]{results/t5_base_topk0_sentence_BLEU_vs_sample_size_1500_2000_2500_t.csv};
    \addplot [name path=upperbound,color=tcolor,no markers,x filter/.expression={and(\thisrow{temperature} == #2, \thisrow{index} == #1) ? x : nan}] table[x=num_decodes,y=mean_sentence_bleu_975p, col sep=comma]{results/t5_base_topk0_sentence_BLEU_vs_sample_size_1500_2000_2500_t.csv};
    \addplot[fill=tcolor, forget plot, fill opacity=0.25] fill between[of=lowerbound and upperbound];
    
    \addplot [name path=lowerbound,color=acolor,no markers,x filter/.expression={and(\thisrow{temperature} == #2, \thisrow{index} == #1) ? x : nan}] table[x=num_decodes,y=mean_sentence_bleu_025p, col sep=comma]{results/t5_base_topk0_sentence_BLEU_vs_sample_size_1500_2000_2500_a.csv};
    \addplot [name path=upperbound,color=acolor,no markers,x filter/.expression={and(\thisrow{temperature} == #2, \thisrow{index} == #1) ? x : nan}] table[x=num_decodes,y=mean_sentence_bleu_975p, col sep=comma]{results/t5_base_topk0_sentence_BLEU_vs_sample_size_1500_2000_2500_a.csv};
    \addplot[fill=acolor, forget plot, fill opacity=0.25] fill between[of=lowerbound and upperbound];
    \end{axis}
\end{tikzpicture}
\end{subfigure}
}

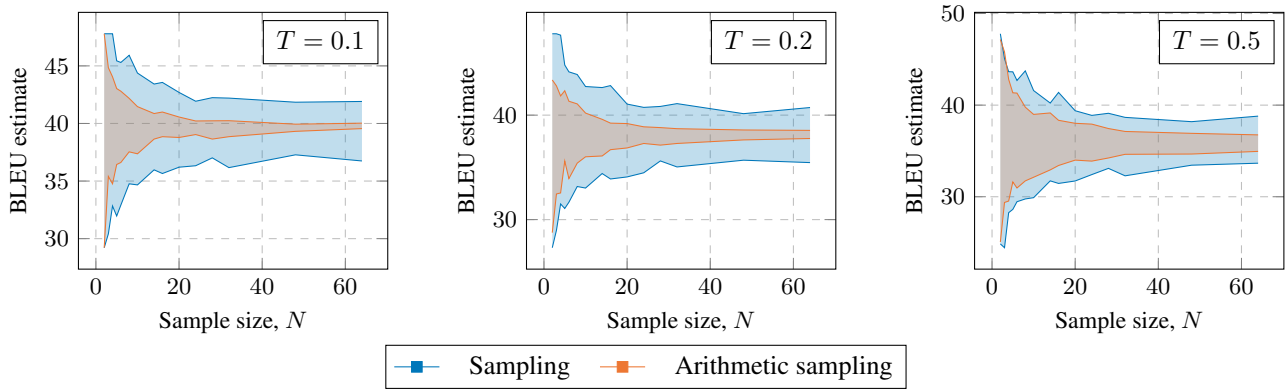
\begin{figure*}
    \centering


    \singleplot{2500}{0.1}{Sample size, $N$}{north east}{}
    \hfill
    \singleplot{2500}{0.2}{Sample size, $N$}{north east}{}
    \hfill
    \singleplot{2500}{0.5}{Sample size, $N$}{north east}{}

    \begin{tikzpicture}
        \begin{customlegend}[
            legend columns=2,
            legend style={
                align=left,
                column sep=2ex
            },
            legend entries={Sampling, Arithmetic sampling}
        ]
            \addlegendimage{mark=square*,solid,color=tcolor}
            \addlegendimage{mark=square*,solid,color=acolor}   
        \end{customlegend}
    \end{tikzpicture}    
    \caption{Sentence BLEU score estimates for a sample sentence (2500 in the test set) estimated using different decoding methods (sampling and arithmetic sampling) for different temperatures $T$. The filled area corresponds to a region from 2.5 to 97.5 percentile of BLEU estimate amongst 100 attempts. We use fine-tuned T5-base model to generate predictions.}
    \label{fig:sentence_bleu_vs_sample_size_1500}
\end{figure*}

\let\singleplot\undefined

\newcommand{\singleplot}[4]{
\begin{subfigure}[t]{0.3\textwidth}
\centering
\begin{tikzpicture}
    \begin{axis}[
    scale=0.6,
    xlabel={Sample size, N},
    ylabel={#4},
    yticklabel={$\pgfmathprintnumber{\tick}$},
    mark=square*,
    ymajorgrids=true,
    xmajorgrids=true,
    grid style=dashed,
    legend pos={#3},
    xlabel style = {font=\small},
    ylabel style = {font=\small},
    ticklabel style = {font=\small}
    ]
    \addlegendimage{empty legend}\addlegendentry{#2}
    
    \addplot [color=tcolor,mark=square*,x filter/.expression={and(\thisrow{datasetid} == #1, \thisrow{methodid} == 0) ? x : nan}] table[x=num_decodes,y=mean_sentence_bleu_std, col sep=comma]{results/t5_base_topk0_metric_sd_vs_sample_size.csv};

    \addplot [color=acolor,mark=square*,x filter/.expression={and(\thisrow{datasetid} == #1, \thisrow{methodid} == 1) ? x : nan}] table[x=num_decodes,y=mean_sentence_bleu_std, col sep=comma]{results/t5_base_topk0_metric_sd_vs_sample_size.csv};

    \end{axis}
\end{tikzpicture}
\end{subfigure}
}

\begin{figure*}
    \centering
    
    \singleplot{0}{WMT14 EnFr}{north east}{Avg. BLEU estimate SD}
    \hfill
    \singleplot{1}{WMT16 EnRo}{north east}{Avg. BLEU estimate SD}
    \hfill
    \singleplot{2}{CNN/DM}{north east}{Avg. ROUGE2 estimate SD}

    \begin{tikzpicture}
        \begin{customlegend}[
            legend columns=3,
            legend style={
                align=left,
                column sep=2ex
            },
            legend entries={Sampling, Arithmetic sampling}            
        ]
            \addlegendimage{mark=square*,solid,color=tcolor}
            \addlegendimage{mark=square*,solid,color=acolor}   
        \end{customlegend}
    \end{tikzpicture}    
    \caption{Average standard deviation of various estimates (sentence BLEU for WMT14 EnFr and WMT16 EnRo, and ROUGE2 for CNN/DailyMail) for the two sampling methods averaged over 100 samples. For each sample we compute estimate 100 times and then calculate standard deviation. 
    }
    \label{fig:metric_sd_vs_sample_size}
\end{figure*}
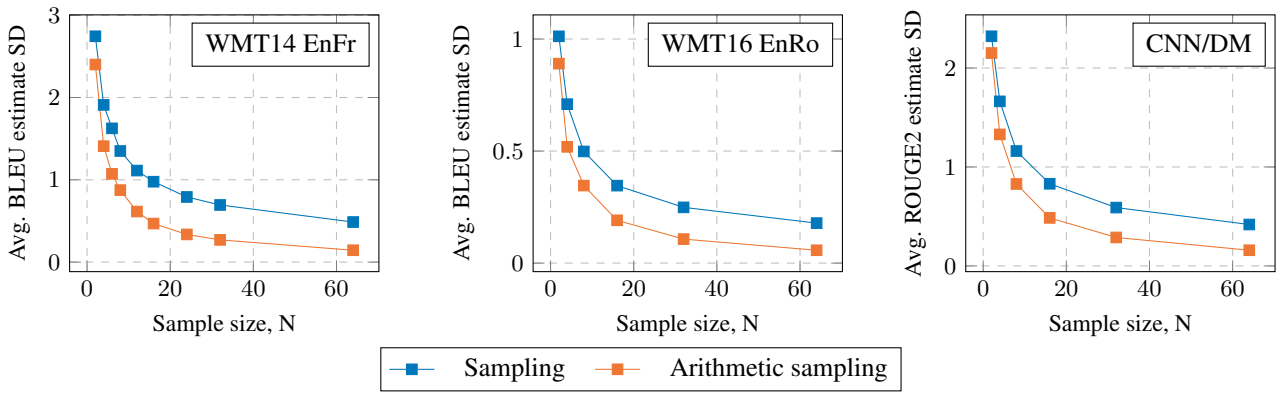

\let\singleplot\undefined
\let\pctreductionplot\undefined

We perform several experiments on sequence-to-sequence models trained for machine translation (WMT14 English-French and WMT16 English-Romanian) and summarization (CNN DailyMail) tasks following a setup similar to \citep{t5}. We use the T5-base model \citep{t5} initialized from a public pre-trained checkpoint, which we fine-tune on each task separately using the T5X library \cite{t5x}. We do 260,000 fine-tuning steps with batch size 128. Using a beam search with beam size 4 leads to 42.39 BLEU score on the WMT14 EnFr test set, 28.34 BLEU score on EnRo, and 20.75 ROUGE-2 on CNN DailyMail.

For the subsequent sampling experiments we run all methods on equivalent accelerator platforms. The computational considerations of running large scale LLM inference with different algorithms possessing different parallelism and synchronization properties are discussed further in Section \ref{sec:parallelization-experiments} and Appendix \ref{app:parallelization}.

\subsection{Diverse beam generation}

We compare arithmetic sampling to normal ancestral sampling on the task of generating multiple diverse translations for a single input sentence. In the context of a real-world system, this kind of diverse generation is often useful when combined with a \emph{reranker} model that can select the best translation from a set of candidates. We control diversity using the softmax temperature parameter $T=0.1,0.2,...,0.8$. Following \cite{kool2019stochastic}, within the set of $K$ translations generated by each sampling method, we take the average, min, and max of the BLEU score against the average $n$-gram diversity score, and then macro average this across the whole test set. The $n$-gram diversity score is defined as $d=\sum_{n=1}^{4} d_n$ where
\begin{align*}
    d_n = \frac{\text{\# of unique $n$-grams in $N$ translations}}{\text{total \# of $n$-grams in $N$ transations}}.
\end{align*}
Results on WMT14 English-French for $N=5,10,20$ are documented in Figure \ref{fig:t5_base_topk0_bleu_vs_diversity}. We can see that for a given temperature setting, arithmetic sampling achieves comparable BLEU score on average to regular sampling, with slightly improved diversity. This is expected, because our sampling method is unbiased and cannot mathematically do better than standard sampling when simply averaging results.

However, when we allow an oracle to pick the highest-scoring element of the beam, arithmetic sampling achieves nearly one additional point of BLEU, showing that the diversity is not just random noise but is actually helping to generate better candidates. This oracle experiment is relevant to the case where we have a reranker and wish to generate the best set of candidates.

We include additional results for WMT16 English-Romanian in Appendix \ref{app:additional-experiments} (see Figure~\ref{fig:t5_base_enro_topk0_bleu_vs_diversity}).

\paragraph{Comparison to beam search} One way to think of arithmetic sampling is as trying to approach the diversity of a search-based method while remaining embarrassingly parallel. For beam search, we found the maximum BLEU score within a beam, averaged over examples, to be 44.15, 47.0, and 49.94, respectively, for beam sizes of 5, 10, and 20. Temperature does not affect beam search, but for ancestral sampling and arithmetic sampling, in order to get a single number across the whole curve of temperatures, we simply took the maximum over all temperatures and beam elements (not necessarily the temperature for which arithmetic sampling most outperforms ancestral sampling), and averaged over examples. For arithmetic sampling, these numbers were 43.20, 46.62, and 49.39, and for ancestral sampling, 42.74, 45.90, and 48.68. This means that arithmetic sampling closes the gap between standard ancestral sampling and beam search by 33\%, 63\%, and 56\% respectively for those beam sizes in this particular oracle setting, without requiring any more synchronization.

\paragraph{Top-k sampling} We conduct experiments on combining arithmetic sampling with the \emph{top-$k$} modification, wherein all conditional probabilities besides the $k$ highest are zero'd out during the ancestral sampling scheme. Moreover, arithmetic sampling is compatible with all such inference schemes that modify conditional logits, such as nucleus sampling and typical decoding. We see that standard sampling and arithmetic sampling perform similarly with the top-$k$ modification for $k \in \{2, 10\}$ as they do without the modification, with arithmetic sampling providing a boost to the maximum BLEU score (full results in Appendix \ref{app:additional-experiments}). 

\subsection{Estimator variance reduction}

In addition to generating diverse beams for reranking, our method is also useful for constructing reduced-variance estimators. We estimate the conditional expected sentence-level BLEU score of the trained translation and ROUGE-2 for summarization models for various input sentences. In addition to demonstrating statistical properties, such sentence-level estimators are useful for training regimes that directly optimize averaged metrics as an expected reward, such as directly optimizing BLEU score for translation models. \citep{sentencebleutraining, 45610}. More recently, training LLMs with expected reward has become especially important in the context of reward functions learned from human annotators, known as RLHF (reinforcement learning from human feedback) \cite{christiano2017deep, rlhfsum}, used in e.g. the InstructGPT model \cite{rlhf}.

Let $M$ be a some sentence-level estimator, for example sentence BLEU or ROUGE-2 score. The expected estimator value for a given sentence $x$ with ground truth reference labels $y^*$ is the conditional expectation
\begin{align*}
    \mathds{E}[\text{M}(Y, y^*)|X=x] = \sum_{y} \text{M}(y,y^*)P(Y = y| X=x),
\end{align*}
where $P(Y = y | X=x)$ is defined by the model. 
We compare traditional ancestral sampling to arithmetic sampling by using each to construct a sample mean estimator for these expectations as in Section \ref{sec:consistency-and-bias}. 

Following \cite{kool2019stochastic} we compute the estimate 100 times for a sample test sentence. The results for WMT14 English-French are shown in Figure~\ref{fig:sentence_bleu_vs_sample_size_1500}. We vary the temperature parameter $T = 0.1, 0.2, 0.5$ and the sample size from 2 to 64.  We report the empirical 2.5-th and 97.5-th percentiles to demonstrate the variance. We present results on three other sentences in Appendix \ref{app:additional-experiments}, and see similar results, most noticeable in the low temperature regime. This is consistent with our theory as the lower temperature regimes put probability mass on a smaller number of higher-BLEU sequences, especially when using a base-sized T5 model that can easily degenerate at higher sampling temperatures, as we see in Figure \ref{fig:t5_base_topk0_bleu_vs_diversity}.

Finally, we average the standard deviations of expected BLEU and ROUGE-2 scores (computed over 100 runs as above) over 100 sample sentences from the WMT and CNN DailyMail test sets, comparing standard sampling with arithmetic sampling. The results are plotted in Figure \ref{fig:metric_sd_vs_sample_size}, and we see that the standard deviation under arithmetic sampling is always lower, and is cut in half once we reach $N \approx 16$ samples.

\begin{figure*}[h!]
    \hspace{-5pt}
        \begin{tikzpicture}[scale=1.0]
            \begin{axis}[
            scale only axis,
            width=0.9\textwidth,
            height=0.42\textwidth,            
            ylabel={Minimal number of chips per replica},
            xlabel={Sample size},
            mark=x,
            xmode=log,
            ymajorgrids=true,
            xmajorgrids=true,
            xminorticks=true,
            grid style=dashed,
            legend columns=2,
            legend cell align=left,
            legend pos={north west},
            ]
            \addplot[color=green,mark size=1pt,line width=2] table {
                1 1 
                2 1
                4 1
                8 1
                16 1
                32 1
                64 2 
                128 2
                256 4
                512 4
                1024 8
            };
            \addplot[color=green,dotted,mark size=1pt,line width=2] table {
                1 1 
                2 1
                4 1
                8 1
                16 1
                32 1
                64 1
                128 2
                256 2
                512 2
                1024 4
            };
            \addplot[color=yellow,mark size=1pt,line width=2] table {
                1 1 
                2 1
                4 1
                8 1
                16 1
                32 2
                64 2 
                128 2
                256 4
                512 8
            };     
            \addplot[color=yellow,dotted,mark size=1pt,line width=2] table {
                1 1 
                2 1
                4 1
                8 1
                16 1
                32 1
                64 2 
                128 2
                256 2
                512 4
            };            
            \addplot[color=red,mark size=1pt,line width=2] table {
                1 1 
                2 1
                4 1
                8 1
                16 2
                32 2
                64 2 
                128 4
                256 8
            };
            \addplot[color=red,dotted,mark size=1pt,line width=2] table {
                1 1 
                2 1
                4 1
                8 1
                16 1
                32 2
                64 2 
                128 2
                256 2
            };
            \addplot[color=blue,mark size=1pt,line width=2] table {
                1 1 
                2 1
                4 2
                8 2
                16 2
                32 4
                64 4 
                128 8
            };
            \addplot[color=blue,dotted,mark size=1pt,line width=2] table {
                1 1 
                2 1
                4 1
                8 2
                16 2
                32 2
                64 4 
                128 4
            };
            \legend{128 len BS,128 len AS,256 len BS,256 len AS,512 len BS,512 len AS,1024 len BS,1024 len AS}        
            \end{axis}
        \end{tikzpicture}  
    \caption{Each solid line represents the minimal \# of TPU v4 chips in a single platform to perform a beam search (BS) with a given number of samples (beam size). Each dotted line represents the minimal \# of TPU v4 chips in a platform replica to perform arithmetic sampling (AS) with a given number of samples, while maintaining the same overall samples/chip efficiency as beam search across replicas.}
    \label{fig:chips_vs_num_decodes}
\end{figure*}
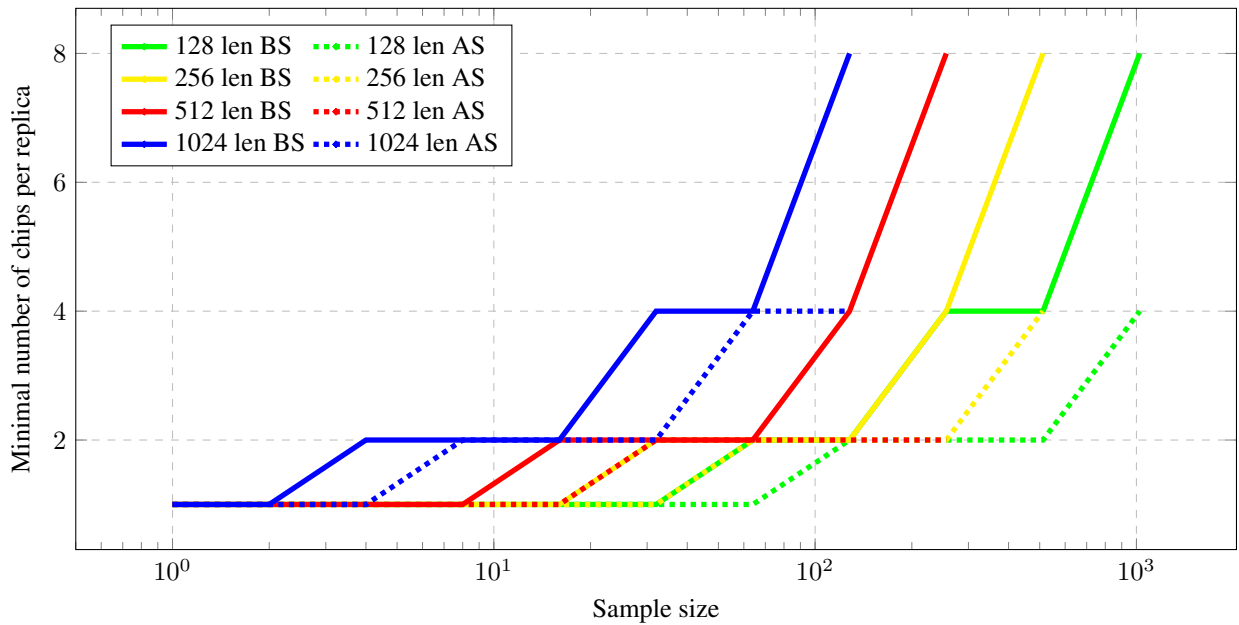

\subsection{Parallelization properties of arithmetic sampling}
\label{sec:parallelization-experiments}
As discussed in more detail in Appendix \ref{app:parallelization}, exactly how arithmetic sampling and beam search compare in terms of parallelism depends on the topology of the underlying hardware setup. To investigate the parallel properties of arithmetic sampling on real hardware, we use the publicly available mt5 XXL model in the T5X library \citep{t5x}, a 13B parameter encoder-decoder model, and between 1 and 8 Google Cloud TPU v4 accelerator chips with 32GB memory each, arranged in a 3D toroidal topology of either 1x1, 1x2x1, 2x2x1, or 2x2x2. For different sizes of models, sample sizes, and other (especially less powerful) accelerators, the exact numbers will differ, but this is a fairly realistic setting at the time of publication. We examine the number of chips in a topology needed to perform beam search or sampling with 1,2,4,8,..,1024 samples, for model input/output sizes of 128, 256, 512, and 1024, up to the full 2x2x2 topology.

In Figure \ref{fig:chips_vs_num_decodes}, we plot the minimum \# of chips in a topology needed to get a given number of samples for each of the 4 input/output sizes, for beam search and arithmetic sampling.

Each of these scenarios permits a sample size of at least 1 to be drawn using only one v4 chip. As the number of samples increases, we plot the minimum platform size needed by beam search on the y axis, as measured by number of chips connected in either 1x1x1, 1x2x1, 2x2x1, or 2x2x2 topologies (1,2,4, or 8). For arithmetic sampling, since it is embarrassingly parallel and we know at least a single sample can be drawn from a 1x1x1, we technically never need more than one chip to draw any number of samples.

However, this requires an entire copy of the models parameters per 1x1x1 replica, and might be wasteful in some applications. In this spirit, Figure \ref{fig:chips_vs_num_decodes} also presents a set of dotted lines below the solid ones, which represent the minimal \# of chips in a topology necessary for arithmetic sampling to generate that \# of samples, while maintaining the same samples/chip as beam search for that situation.

\section{Conclusion and Future Work}

We introduce a method for parallel sampling from large language models that improves beam diversity and retains unbiased expectations from the original model. It is simple to implement and compatible with many existing modifications to LLM inference. Our method, based on creating a code space that makes it easier to search through sequences, opens up many avenues for future work. 

We are most excited to explore the construction of more sophisticated multi-dimensional coding schemes which are still structured to allow easy sampling, while capturing even more geometric structure latent in the space of sequences, including the construction of coding schemes to improve inference with respect to a given reward function. 

\subsubsection*{Acknowledgements}
We thank David Belanger and the anonymous reviewers for their helpful comments on the paper, as well as Zach Fisher, Santiago Onta{\~n}{\`o}n, Bhargav Kanagal, and Sudeep Gandhe for conversations over the course of the product. We would also like to thank Sharad Vikram and the rest of the JAX and T5X teams who helped with advice and debugging.

\bibliography{paper.bib}
\bibliographystyle{icml2023}

\newpage
\onecolumn
\appendix
\begin{center}
\Large{Appendix}
\end{center}

\section{Computational considerations in parallel inference for LLMs}
\label{app:parallelization}

In discussions of the relative parallelizability of different inference algorithms for LLMs (e.g. our claims that arithmetic sampling is parallel while beam search is not), the careful reader might notice that even beam search itself has many parallel components in its computation. Indeed, at each time step, predictions of the next token can be done for each element of the beam in parallel. However, while it is parallel in this sense, it is not \emph{embarrassingly parallel}, requiring a synchronization step between all beams after predicting each token, in order to sort and broadcast the new beams (called \emph{all-gather}).

Modern large models are spread across 10s, 100s, or even 1000s of accelerator chips connected by very fast memory interconnects, essentially forming a large and topologically complex supercomputer (a \emph{platform}). Bigger models or bigger computation graphs require bigger platforms. Whether something is a beam search which must all happen in shared memory with a synchronization between each step, or an embarrassingly parallel algorithm like arithmetic sampling, it will all happen mostly in parallel, the question is really about the size of platform which it can be used on, which itself is determined by the frequency of synchronizations, because repeated fine-grained synchronizing across the network is prohibitively slow for LLM inference. These concerns also apply to other search-based inference methods such as Gumbel top-k sampling.

Large platforms can be very hard to schedule and very expensive, while embarrassingly parallel jobs can happen on small and cheap platforms. Hence, the motivation for producing embarrassingly parallel algorithms is not just about the speed with which something can be computed, but whether it can be practically computed at all on available hardware. To compute a very large beam search with any standard implementation might require an impractically large platform, whereas any number of samples (or arithmetic sampling samples) can be computed as long as we have enough small platforms available, because we simply need to distribute the RNG seeds (and/or codes) once at the beginning and then proceed completely independently.

\section{Variance for step functions}
\label{app:variance-rqmc}

We can analyze by looking at the contributions to the variance of the estimator. The variance of a sample mean estimator ${\hat \mu} = \frac{1}{N} \sum_i f(c_i)$ over $N+1$ samples $c_i = u + i/N ~\text{mod}~1$ can be decomposed as
\begin{align*}
    \text{Var}[{\hat \mu}] = \frac{\text{Var}[f(u)]}{N+1} + \frac{1}{(N+1)^2} \sum_{i \ne j} \text{Cov}[f(c_i), f(c_j)]
\end{align*}
In order for this variance to be lower than the naive estimator, the covariance term has to be negative. Intuitively, when integrating indicator functions of intervals, this term should be negative, since the values of $f$ for lattice points within the interval are greater than the mean while the values outside are less, and distant lattice points are anticorrelated in terms of which bucket they can fall into. The values for lattice points in the same bucket are co-vary positively though, so we have to ensure the former is a bigger factor than the latter carefully.

Plugging our formula for step functions with $N$ components into this covariance term we get
\begin{align*}
    \sum_{i \ne j} \text{Cov}[\sum_k^N a_k \mathds{1}_{[x_k^\wedge, x_k^\vee]}(c_i), \sum_l^N a_l \mathds{1}_{[x_l^\wedge, x_l^\vee]}(c_j)] =\\
    ~~~~\sum_{i \ne j} \sum_k^N \sum_l^N a_k a_l \text{Cov}[\mathds{1}_{[x_k^\wedge, x_k^\vee]}(c_i), \mathds{1}_{[x_l^\wedge, x_l^\vee]}(c_j)]
\end{align*}
Now, assuming that each interval has a width that is a multiple of $1/N$, without loss of generality, we can analyze the terms corresponding to the entire pair of indicator functions by analyzing what happens with a function of the form
\begin{align*}
    f(c) = \mathds{1}_{[0, \frac{1}{N}]}(c) + \mathds{1}_{[\frac{1}{N}, \frac{2}{N}]}(c).
\end{align*}
The full sum for general step functions will follow by a bit of matrix algebra.
So we need to examine the terms of
\begin{align*}
    \sum_{i \ne j} \big( &\text{Cov}[\mathds{1}_{[0, \frac{1}{N}]}(c_i),  \mathds{1}_{[0, \frac{1}{N}]}(c_j)] + \text{Cov}[\mathds{1}_{[\frac{1}{N}, \frac{2}{N}]}(c_i),  \mathds{1}_{[\frac{1}{N}, \frac{2}{N}]}(c_j)]\\ &+ \text{Cov}[\mathds{1}_{[0, \frac{1}{N}]}(c_i),  \mathds{1}_{[\frac{1}{N}, \frac{2}{N}]}(c_j)] + \text{Cov}[  \mathds{1}_{[\frac{1}{N}, \frac{2}{N}]}(c_i), \mathds{1}_{[0, \frac{1}{N}}(c_j)] \big)
\end{align*}
By symmetry, the first and second terms have the same value, as do the third and fourth, and in fact, with any number of such functions the pairwise terms will always look one way and the self interaction terms the other way. So we are concerned with the values of the expressions
\begin{align}
\label{exp:full-cov-diagonal}
    \sum_{i \ne j} \text{Cov}[\mathds{1}_{[0, \frac{1}{N}]}(c_i),  \mathds{1}_{[0, \frac{1}{N}]}(c_j)]
\end{align}
and
\begin{align}
\label{exp:full-cov-off-diagonal}
    \sum_{i \ne j} \text{Cov}[\mathds{1}_{[0, \frac{1}{N}]}(c_i),  \mathds{1}_{[\frac{1}{N}, \frac{2}{N}]}(c_j)]
\end{align}
We are concerned about the covariance structure of the lattice points, so we examine their behavior jointly. Let's name the intervals as following
\begin{align*}
    L = [0, \frac{1}{N}], R = [\frac{1}{N}, \frac{2}{N}], O = [\frac{2}{N}, 1], E = [\frac{1}{N}, 1]
\end{align*}
For a randomly shifted lattice of $N+1$ points, three different things can happen. Either the extra lattice point shows up in $L$, $R$, or $O$. Since the marginal distribution of a lattice point is uniform, this occurs with probabilities $1/N$, $1/N$, and $(N-2)/N$.
\begin{align*}
    p_1 = \frac{1}{N}, p_2 = \frac{1}{N}, p_3 = \frac{N-2}{N}
\end{align*}
Let $l_i, r_i, o_i, e_i$ be the counts of points in each interval in each scenario (so $e_i = r_i+o_i$), so we have
\begin{align*}
    &(l_1, r_1, o_1, e_1) = (2, 1, N - 2, N-1)\\ 
    &(l_2, r_2, o_2, e_2) = (1, 2, N - 2, N)\\ 
    &(l_3, r_3, o_3, e_3) = (1, 1, N - 1, N)\\ 
\end{align*}
Now we can examine the sum in expression \ref{exp:full-cov-diagonal}, summed over all unequal pairs of variables. Since we are only looking at the first indicator, we are concerned with quantities involving $L$ and $E$ (either within the support of the first indicator, or without it). 

For a given pair of distinct points $(c_i, c_j)$, they are each either in $L$ or $E$. The number of such pairs depends on the configuration $(l,r,o,e)$. Call the number of such pairs $m_{LE}$, etc.
\begin{align*}
    m_{LL} = l(l-1), m_{LE} = m_{EL} = le = l(r+o), m_{EE} = e(e-1) = (r+o)(r+o-1)
\end{align*}
Finally we have the summand, the actual covariance term between the two variables $(c_i, c_j)$. The expectation of the function $\mathds{1}_{[0, \frac{1}{N}]}(c)$ is $\mu_f = \frac{1}{N}$, so the covariance terms for pairs of points look like one of three ways, either the points are both in the support of an indicator, both outside, or mixed.
\begin{align*}
    &s_{LL} = (1 - \frac{1}{N})^2 = \frac{(N-1)^2}{N^2} = s_1\\
    &s_{LE} = s_{EL} = (1 - \frac{1}{N})(0 - \frac{1}{N}) = \frac{1 - N}{N^2} = s_2\\
    &s_{EE} = (0 - \frac{1}{N})^2 = \frac{1}{N^2} = s_3
\end{align*}

So the the sum in expression \ref{exp:full-cov-diagonal} is equal to
\begin{align*}
    &\sum_{i = 1}^3 p_i(m_{LL, i}s_1 + 2m_{LE, i}s_2 + m_{EE, i}s_3)\\
    =& \sum_{i = 1}^3 p_i(l_i(l_i-1)s_1 + 2 l_i(r_i+o_i)s_2 + (r_i+o_i)(r_i+o_i-1)s_3 )\\
    =& \frac{1}{N} - 1 = c_{\text{on}}
\end{align*}
To analyze expression \ref{exp:full-cov-off-diagonal}, we can proceed in a similar way. For a given pair of distinct points $(c_i, c_j)$, they are each either in $L$, $R$, or $O$. The number of such pairs depends on the configuration $(l,r,o,e)$.
\begin{align*}
    &m_{LL} = l(l-1), m_{RR} = r(r-1), m_{OO} = o(o-1) \\
    &m_{LR} = m_{RL} = lr, m_{LO} = m_{OL} = lo, m_{RO} = m_{OR} = ro 
\end{align*}
And the covariance terms look like
\begin{align*}
    &s_{LL} = s_{LO} = s_{RR} = s_{OR} = \frac{1 - N}{N^2} = s_2\\
    &s_{LR} = \frac{(N-1)^2}{N^2} = s_1\\
    &s_{OO} = s_{RL} = s_{OL} = s_{RO} = \frac{1}{N^2} = s_3
\end{align*}
So expression \ref{exp:full-cov-off-diagonal} is equal to
\begin{align*}
    &\sum_{i = 1}^3 p_i(m_{LL, i}s_2 + m_{RR, i}s_2 + m_{OO, i}s_3 + m_{LR, i}s_1 + m_{RL, i}s_3 + m_{LO, i}s_2 + m_{OL, i}s_3 + m_{RO, i}s_3  + m_{OR, i}s_2)\\
    = &\sum_{i = 1}^3 p_i(m_{LR, i}s_1 + (m_{LL, i} + m_{RR, i} + m_{LO, i} + m_{OR, i})s_2 + (m_{OO, i} + m_{RL, i} + m_{OL, i} + m_{RO, i})s_3)\\
    = &\sum_{i = 1}^3 p_i(lrs_1 + (l(l-1) + r(r-1) + lo + ro)s_2 + (o(o-1) + lr + lo + ro)s_3)\\
    = & \frac{1}{N} = c_{\text{off}}
\end{align*}
Now we turn to examine the full sum
\begin{align}
\label{exp:full-sum-cov}
    \sum_{i \ne j} \sum_k^N \sum_l^N a_k a_l \text{Cov}[\mathds{1}_{[x_k^\wedge, x_k^\vee]}(c_i), \mathds{1}_{[x_l^\wedge, x_l^\vee]}(c_j)]
\end{align}
To show that this is nonpositive for any vector of coefficients $a_k$, we need to show that the matrix $C$ with $c_\text{off}$ on the off-diagonal terms and $c_\text{on}$ for the diagonal terms is negative semidefinite.
\begin{align*}
    C = c_\text{off} \mathds{1} \mathds{1}^\top + (c_\text{on}-c_\text{off})I.
\end{align*}
The first term of this summand has one eigenvalue equal to $\frac{N-1}{N}$ and $N-1$ eigenvalues equal to 0. Since the second term of the summand is a spherical matrix, it shifts all the eigenvalues by the constant amount $c_\text{on}-c_\text{off}$. So $C$ has one eigenvalue equal to $\frac{N-1}{N} - 1 + \frac{1}{N} = 0$ and the rest equal to $-1 + \frac{1}{N} - \frac{1}{N} = -1$. 

So $C$ is negative semidefinite, as required, and the sum in expression \ref{exp:full-sum-cov} is nonpositive for any vector of coefficients $a_k$. The eigenvector corresponding to the zero eigenvalue has an intuitive interpretation as corresponding to the vector $a\mathds{1}$, the constant function with zero variance and thus zero covariance.

\section{Bias and consistency}
\label{app:bias-consistency}

\begin{proof}[Proof of Proposition \ref{prop:naive-estimator}]
Since we know that the pushforward measure induced by the random variable $f$ applying Algorithm \ref{alg:sampling-from-code} on a uniform measure on the unit interval gives the distribution $P(X)$, we can write the RHS of equation \ref{eqn:naive-estimator} as
\begin{align*}
\mathds{E}[s(X)] = \int_0^1s(f(c))dc,
\end{align*}
and since $s(f(c))$ is a step function (a linear combination of indicators of preimages in $[0,1]$ for each sequence $x$), it is Riemann integrable, and the LHS is a Riemann sum with mesh size $1/(N+1)$ for the integral.
\end{proof}

\begin{proof}[Proof of Proposition \ref{prop:final-estimator-unbiased}]
The consistency of the estimator can be proven in the same way as naive arithmetic sampling, since it is also a valid Riemann sum with mesh size $1/(N+1)$, no matter what offset is sampled. To show unbiasedness, note that for any constant $c$, the marginal distribution of $c + u ~~ \text{mod} ~~ 1$ for $u \sim U(0,1)$ is just $U(0,1)$ again, the rest proceeds by linearity of expectation.
\end{proof}

\section{Additional Experiments}
\label{app:additional-experiments}
\newcommand{\singleplot}[2]{
\begin{subfigure}[t]{0.3\textwidth}
\centering
\begin{tikzpicture}
    \begin{axis}[
    scale=0.65,
    xlabel={$n$-gram diversity},
    ylabel={BLEU},
    mark size=2pt,
    solid,
    ymajorgrids=true,
    xmajorgrids=true,
    grid style=dashed,
    legend pos=south west,
    xlabel style = {font=\small},
    ylabel style = {font=\small},
    ticklabel style = {font=\small}
    ]
    \addlegendimage{empty legend}\addlegendentry{$\text{TOPK} = #1, N = #2$}
    \addplot [color=tcolor,mark=square*,x filter/.expression={and(\thisrow{topk} == #1, \thisrow{num_decodes} == #2) ? x : nan}] table[x=4gram_diversity_macro,y=mean_sentence_bleu, col sep=comma]{results/t5_base_topk2_topk10_topk0_bleu_vs_diversity_t.csv};
    \addplot [color=tcolor,mark=triangle*,mark size=3pt,x filter/.expression={and(\thisrow{topk} == #1, \thisrow{num_decodes} == #2) ? x : nan}] table[x=4gram_diversity_macro,y=max_sentence_bleu, col sep=comma]{results/t5_base_topk2_topk10_topk0_bleu_vs_diversity_t.csv};
    
    \addplot [color=acolor,mark=triangle*,x filter/.expression={and(\thisrow{topk} == #1, \thisrow{num_decodes} == #2) ? x : nan}] table[x=4gram_diversity_macro,y=mean_sentence_bleu, col sep=comma]{results/t5_base_topk2_topk10_topk0_bleu_vs_diversity_a.csv};
    \addplot [color=acolor,mark=triangle*,mark size=3pt,x filter/.expression={and(\thisrow{topk} == #1, \thisrow{num_decodes} == #2) ? x : nan}] table[x=4gram_diversity_macro,y=max_sentence_bleu, col sep=comma]{results/t5_base_topk2_topk10_topk0_bleu_vs_diversity_a.csv};
    
    \end{axis}
\end{tikzpicture}
\end{subfigure}
}

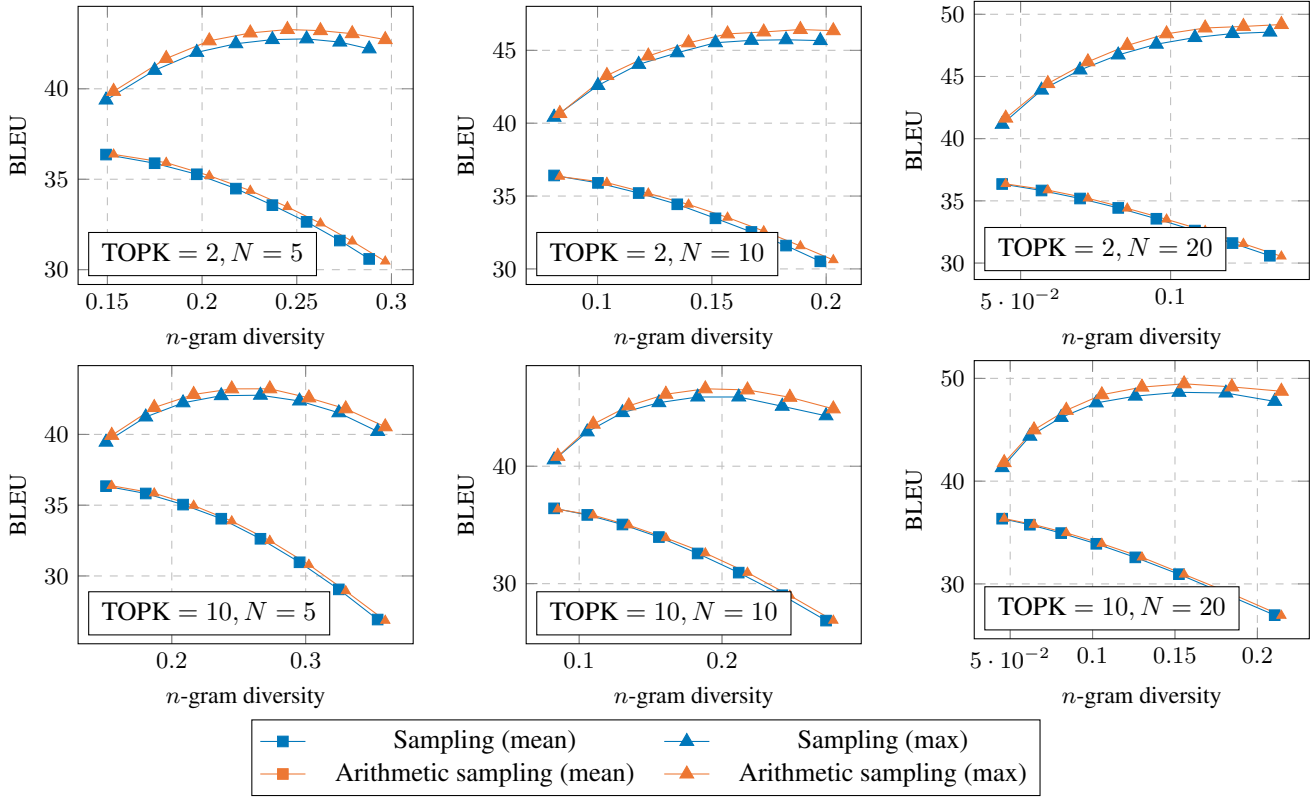
\begin{figure*}
    \centering
    \singleplot{2}{5}
    \hfill
    \singleplot{2}{10}
    \hfill
    \singleplot{2}{20}
    
    \singleplot{10}{5}
    \hfill
    \singleplot{10}{10}
    \hfill
    \singleplot{10}{20}

    \begin{tikzpicture}
        \begin{customlegend}[
            legend columns=2,
            legend style={
                align=left,
                column sep=2ex
            },
            legend entries={Sampling (mean), Sampling (max), Arithmetic sampling (mean), Arithmetic sampling (max)}
        ]
            \addlegendimage{mark=square*,solid,color=tcolor}
            \addlegendimage{mark=triangle*,mark size=3pt,solid,color=tcolor}
            \addlegendimage{mark=square*,solid,color=acolor}
            \addlegendimage{mark=triangle*,mark size=3pt,solid,color=acolor}
        \end{customlegend}
    \end{tikzpicture}    
    \caption{BLEU score vs ngram diversity.}
    \label{fig:t5_base_topk2_topk10_bleu_vs_diversity}
\end{figure*}

\let\singleplot\undefined

\newcommand{\singleplot}[5]{
\begin{subfigure}[t]{0.3\linewidth}
\centering
\begin{tikzpicture}
    \begin{axis}[
    scale=0.65,
    xlabel={$n$-gram diversity},
    ylabel={BLEU},
    mark size=2pt,
    solid,
    ymajorgrids=true,
    xmajorgrids=true,
    grid style=dashed,
    legend pos=south west,
    xlabel style = {font=\small},
    ylabel style = {font=\small},
    ticklabel style = {font=\small},
    ]
    \addlegendimage{empty legend}\addlegendentry{$N = #2$}
    \addplot [mark=none, color=black, dotted, line width=2] table[row sep=crcr] {
    #3 #5\\
    #4 #5\\
    };
    \addplot [color=tcolor,mark=square*,x filter/.expression={and(\thisrow{topk} == #1, \thisrow{num_decodes} == #2) ? x : nan}] table[x=4gram_diversity_macro,y=mean_sentence_bleu, col sep=comma]{results/t5_base_enro_topk0_bleu_vs_diversity_t.csv};
    \addplot [color=tcolor,mark=triangle*,mark size=3pt,x filter/.expression={and(\thisrow{topk} == #1, \thisrow{num_decodes} == #2) ? x : nan}] table[x=4gram_diversity_macro,y=max_sentence_bleu, col sep=comma]{results/t5_base_enro_topk0_bleu_vs_diversity_t.csv};
    
    \addplot [color=acolor,mark=square*,x filter/.expression={and(\thisrow{topk} == #1, \thisrow{num_decodes} == #2) ? x : nan}] table[x=4gram_diversity_macro,y=mean_sentence_bleu, col sep=comma]{results/t5_base_enro_topk0_bleu_vs_diversity_a.csv};
    \addplot [color=acolor,mark=triangle*,mark size=3pt,x filter/.expression={and(\thisrow{topk} == #1, \thisrow{num_decodes} == #2) ? x : nan}] table[x=4gram_diversity_macro,y=max_sentence_bleu, col sep=comma]{results/t5_base_enro_topk0_bleu_vs_diversity_a.csv};
    \end{axis}
\end{tikzpicture}
\end{subfigure}
}

\begin{figure*}
    \centering
    \singleplot{0}{5}{0.15}{0.33}{27.23207212214543}
    \hfill
    \singleplot{0}{10}{0.1}{0.25}{29.502577211064285}
    \hfill
    \singleplot{0}{20}{0.05}{0.17}{31.78855431827144}

    \begin{tikzpicture}
        \begin{customlegend}[
            legend columns=3,
            legend style={
                align=left,
                column sep=2ex
            },
            legend entries={Sampling (mean), Sampling (max), Arithmetic sampling (mean), Arithmetic sampling (max), Beam search}
        ]
            \addlegendimage{mark=square*,solid,color=tcolor}
            \addlegendimage{mark=triangle*,mark size=3pt,solid,color=tcolor}
            \addlegendimage{mark=square*,solid,color=acolor}
            \addlegendimage{mark=triangle*,mark size=3pt,solid,color=acolor}
            \addlegendimage{mark=none,line width=2,dotted,color=black}
        \end{customlegend}
    \end{tikzpicture}    
    \caption{BLEU score vs $n$-gram diversity for WMT16 English-Romanian for various beam sizes $N$.}
    \label{fig:t5_base_enro_topk0_bleu_vs_diversity}
\end{figure*}
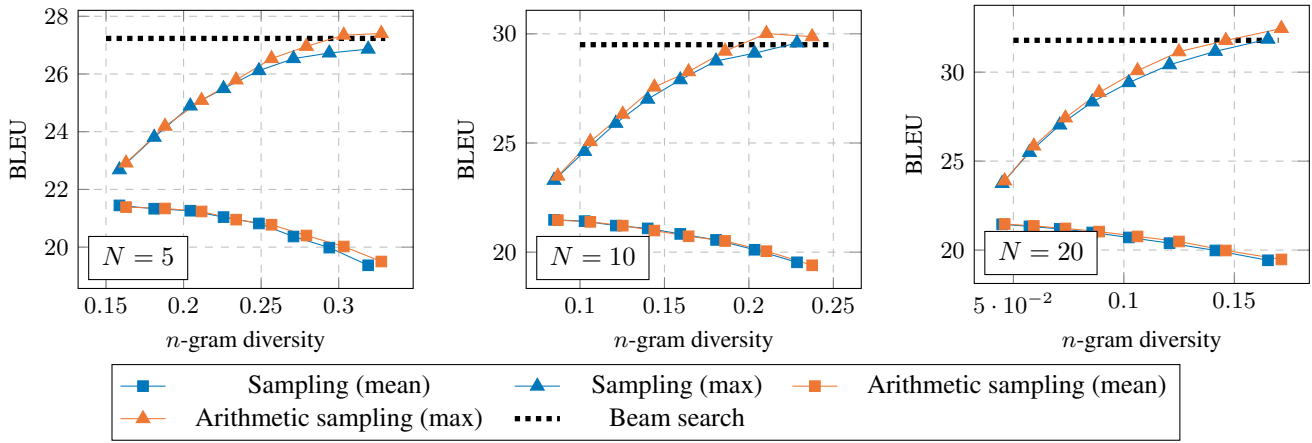

\let\singleplot\undefined

\subsection{WMT English-Romanian sampling diversity}
We show additional results for the WMT English-Romanian test set in Figure~\ref{fig:t5_base_enro_topk0_bleu_vs_diversity}.

\subsection{Arithmetic sampling with top-k modification}
Further results of BLEU vs. $n$-gram diversity for the WMT English-French test set using arithmetic sampling with the top-k sampling modification are shown in Figure \ref{fig:t5_base_topk2_topk10_bleu_vs_diversity}.

\subsection{Arithmetic sampling with temperature modification}
\newcommand{\singleplot}[3]{
\begin{subfigure}[t]{0.3\textwidth}
\centering
\begin{tikzpicture}
    \begin{axis}[
    scale=0.6,
    xlabel={#2},
    ylabel={Avg. BLEU estimate SD},
    yticklabel={$\pgfmathprintnumber{\tick}$.0},
    mark=square*,
    ymajorgrids=true,
    xmajorgrids=true,
    grid style=dashed,
    legend pos={#3},
    xlabel style = {font=\small},
    ylabel style = {font=\small},
    ticklabel style = {font=\small}
    ]
    \addlegendimage{empty legend}\addlegendentry{$T = #1$}
    
    \addplot [color=tcolor,mark=square*,x filter/.expression={\thisrow{temperature} == #1 ? x : nan}] table[x=num_decodes,y=mean_sentence_bleu_std, col sep=comma]{results/t5_base_topk0_sentence_BLEU_sd_vs_sample_size_t.csv};

    \addplot [color=acolor,mark=square*,x filter/.expression={\thisrow{temperature} == #1 ? x : nan}] table[x=num_decodes,y=mean_sentence_bleu_std, col sep=comma]{results/t5_base_topk0_sentence_BLEU_sd_vs_sample_size_a.csv};

    \end{axis}
\end{tikzpicture}
\end{subfigure}
}

\newcommand{\pctreductionplot}[3]{
\begin{subfigure}[t]{0.3\linewidth}
\centering
\begin{tikzpicture}
    \begin{axis}[
    scale=0.6,
    xlabel={#2},
    ylabel={Avg. BLEU est. SD ratio},
    mark=square*,
    ymajorgrids=true,
    xmajorgrids=true,
    grid style=dashed,
    legend pos={#3},
    xlabel style = {font=\small},
    ylabel style = {font=\small},
    ticklabel style = {font=\small}
    ]
    \addlegendimage{empty legend}\addlegendentry{$T = #1$}
    
    \addplot [color=xcolor,mark=square*,x filter/.expression={\thisrow{temperature} == #1 ? x : nan}] table[x=num_decodes,y=mean_sentence_bleu_std_reduction_pct, col sep=comma]{results/t5_base_topk0_sentence_BLEU_sd_vs_sample_size_r.csv};

    \end{axis}
\end{tikzpicture}
\end{subfigure}
}

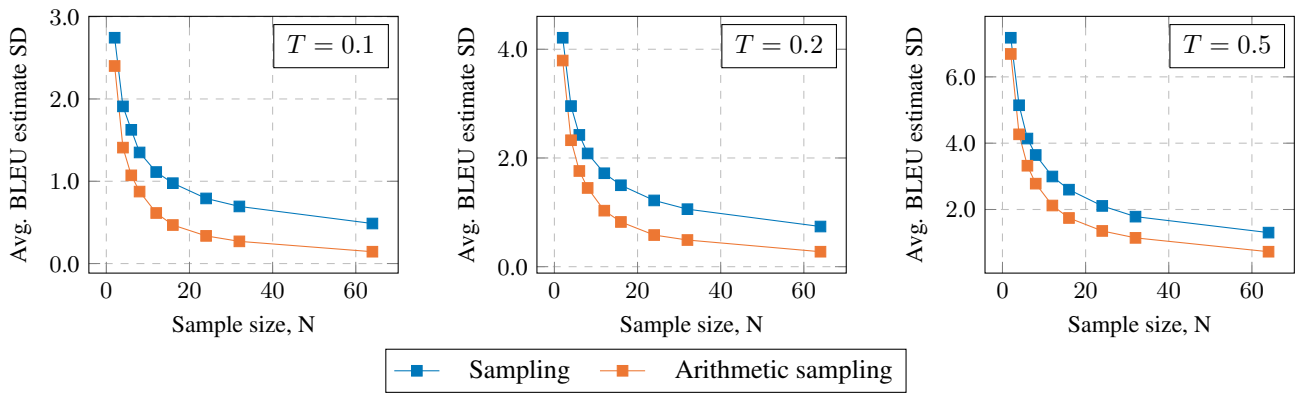
\begin{figure*}
    \centering
    
    \singleplot{0.1}{Sample size, N}{north east}
    \hfill
    \singleplot{0.2}{Sample size, N}{north east}
    \hfill
    \singleplot{0.5}{Sample size, N}{north east}
    
    

    \begin{tikzpicture}
        \begin{customlegend}[
            legend columns=3,
            legend style={
                align=left,
                column sep=2ex
            },
            legend entries={Sampling, Arithmetic sampling}            
        ]
            \addlegendimage{mark=square*,solid,color=tcolor}
            \addlegendimage{mark=square*,solid,color=acolor}   
        \end{customlegend}
    \end{tikzpicture}    
    \caption{Average standard deviation of sentence BLEU estimate for the two sampling methods and different temperatures $T$ averaged over 100 samples. For each sample we estimate sentence BLEU 100 times and compute standard deviation. 
    }
    \label{fig:sentence_bleu_sd_vs_sample_size}
\end{figure*}

\let\singleplot\undefined
\let\pctreductionplot\undefined
\newcommand{\singleplot}[3]{
\begin{subfigure}[t]{0.3\textwidth}
\centering
\begin{tikzpicture}
    \begin{axis}[
    scale=0.65,
    xlabel={#3},
    ylabel={BLEU estimate},
    mark=x,
    ymajorgrids=true,
    xmajorgrids=true,
    grid style=dashed,
    xlabel style = {font=\small},
    ylabel style = {font=\small},
    ticklabel style = {font=\small}
    ]
    \addlegendimage{empty legend}\addlegendentry{$T = #2$}
    \addplot [name path=lowerbound,color=tcolor,no markers,x filter/.expression={and(\thisrow{temperature} == #2, \thisrow{index} == #1) ? x : nan}] table[x=num_decodes,y=mean_sentence_bleu_025p, col sep=comma]{results/t5_base_topk0_sentence_BLEU_vs_sample_size_1501_2001_2501_t.csv};
    \addplot [name path=upperbound,color=tcolor,no markers,x filter/.expression={and(\thisrow{temperature} == #2, \thisrow{index} == #1) ? x : nan}] table[x=num_decodes,y=mean_sentence_bleu_975p, col sep=comma]{results/t5_base_topk0_sentence_BLEU_vs_sample_size_1501_2001_2501_t.csv};
    \addplot[fill=tcolor, forget plot, fill opacity=0.25] fill between[of=lowerbound and upperbound];
    
    \addplot [name path=lowerbound,color=acolor,no markers,x filter/.expression={and(\thisrow{temperature} == #2, \thisrow{index} == #1) ? x : nan}] table[x=num_decodes,y=mean_sentence_bleu_025p, col sep=comma]{results/t5_base_topk0_sentence_BLEU_vs_sample_size_1501_2001_2501_a.csv};
    \addplot [name path=upperbound,color=acolor,no markers,x filter/.expression={and(\thisrow{temperature} == #2, \thisrow{index} == #1) ? x : nan}] table[x=num_decodes,y=mean_sentence_bleu_975p, col sep=comma]{results/t5_base_topk0_sentence_BLEU_vs_sample_size_1501_2001_2501_a.csv};
    \addplot[fill=acolor, forget plot, fill opacity=0.25] fill between[of=lowerbound and upperbound];
    \end{axis}
\end{tikzpicture}
\end{subfigure}
}

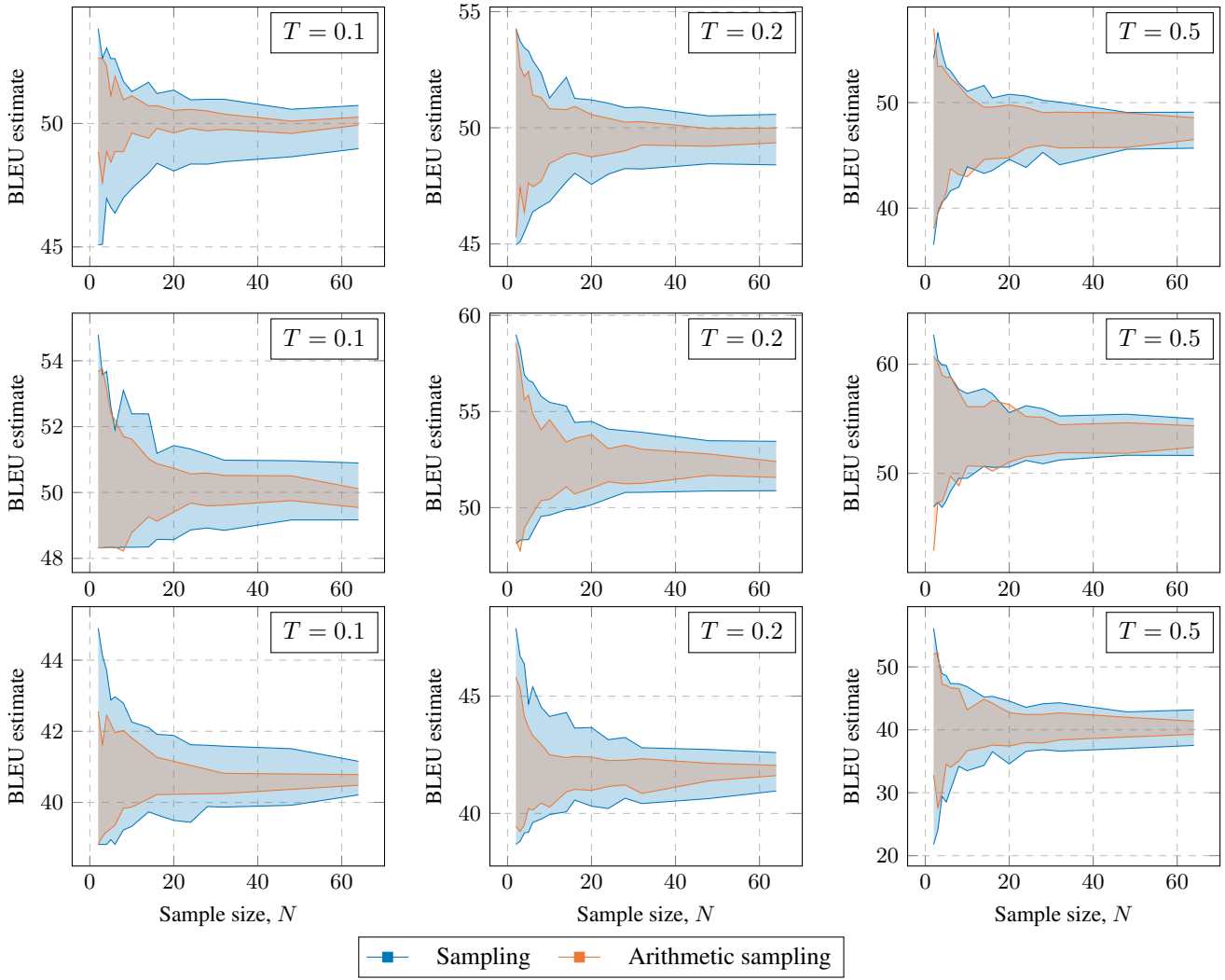
\begin{figure*}
    \centering
    \singleplot{1501}{0.1}{}
    \hfill
    \singleplot{1501}{0.2}{}
    \hfill
    \singleplot{1501}{0.5}{}

    \singleplot{2001}{0.1}{}
    \hfill
    \singleplot{2001}{0.2}{}
    \hfill
    \singleplot{2001}{0.5}{}

    \singleplot{2501}{0.1}{Sample size, $N$}
    \hfill
    \singleplot{2501}{0.2}{Sample size, $N$}
    \hfill
    \singleplot{2501}{0.5}{Sample size, $N$}

    \begin{tikzpicture}
        \begin{customlegend}[
            legend columns=2,
            legend style={
                align=left,
                column sep=2ex
            },
            legend entries={Sampling, Arithmetic sampling}
        ]
            \addlegendimage{mark=square*,solid,color=tcolor}
            \addlegendimage{mark=square*,solid,color=acolor}   
        \end{customlegend}
    \end{tikzpicture}    
    \caption{Sentence BLEU score estimates for three different sentences (1501, 2001 and 2501) estimated using different decoding methods (sampling and arithmetic sampling) for different temperatures $T$. The filled area corresponds to a region from 2.5 to 97.5 percentile of BLEU estimate amongst 100 attempts. We use fine-tuned T5-base model to generate predictions.}
    \label{fig:sentence_bleu_vs_sample_size_1501}
\end{figure*}

\let\singleplot\undefined

We verify that arithmetic sampling reduces variance of sentence BLEU estimation for various sampling temperature levels. We repeat the experiment from Figure~\ref{fig:metric_sd_vs_sample_size} for WMT English-French task, now varying the temperature parameter. Figure~\ref{fig:sentence_bleu_sd_vs_sample_size} shows that the standard deviation of sentence BLEU estimation is always lower with arithmetic sampling regardless of the temperature used. 

\subsection{Sentence BLEU estimation variance}
Results on sentence BLEU estimation variance for 3 other arbitrary WMT English-French test set sentences are shown in Figure \ref{fig:sentence_bleu_vs_sample_size_1501}.

\end{document}